\documentclass[11pt,a4paper,twoside]{article} 

\usepackage[round]{natbib}
\usepackage{graphicx}
\usepackage[latin1]{inputenc}
\usepackage{amsmath,amssymb,amsthm}
\usepackage{xcolor}
\usepackage{ifthen}
\numberwithin{equation}{section}

\newif\ifRR\RRtrue
\RRfalse
\ifRR
\usepackage{RR,RRthemes}
\usepackage[french,english]{babel}
\fi

\newtheorem{prop}{Proposition}[section]

\newtheorem{theo}[prop]{Theorem}
\newtheorem{lem}[prop]{Lemma}
\theoremstyle{definition}
\newtheorem{defin}[prop]{Definition}

\newcommand{\I}{\mathbb{I}}

\renewcommand{\v}{\>v}
\renewcommand{\u}{\>u}
\newcommand{\x}{\>x}
\newcommand{\y}{\>y}
\newcommand{\F}{\mathbb{F}}
\newcommand{\V}{\mathbb{V}}

\newcommand{\W}{\mathcal{W}}
\newcommand{\N}{\mathcal{N}}
\newcommand{\R}{\mathbb{R}}
\newcommand{\E}{\mathbb{E}}
\newcommand{\G}{\mathcal{G}}
\renewcommand{\P}{\mathcal{P}}
\newcommand{\n}{^{\scriptscriptstyle (n)}}

\newcommand{\1}{\leavevmode\hbox{\rm \small1\kern-0.35em\normalsize1}}
\newcommand{\ind}[1]{\1_{\{#1\}}}

\def\egaldef{\stackrel{\mbox{\tiny def}}{=}} 
\def\DD{\displaystyle}

\newcommand{\markchanges}[5][]{%
    \ifthenelse{\equal{#1}{}\or\equal{#2}{}}{}{{\color{yellow} (#1): }}{{\color{blue}#2}}
    \ifthenelse{\equal{#3}{}}{}{{\color{yellow}[Remove (#1): ``#3'']}}
    \ifthenelse{\equal{#4}{}}{}{{\color{green!70!black}[Comment (#1): #4]}}
    \ifthenelse{\equal{#5}{}}{}{{\color{blue}[Comment (#1): #5]}}
}

\begin{document}

\ifRR
\else

\title{The Role of Normalization in the Belief Propagation
Algorithm\thanks{This work was supported by the French National
  Research Agency (ANR) grant Travesti (ANR-08-SYSC-017)}}

\author{Victorin Martin\\INRIA Paris-Rocquencourt
  \and Jean-Marc Lasgouttes\\INRIA Paris-Rocquencourt
  \and Cyril Furtlehner\\INRIA Saclay
}

\maketitle

\begin{abstract}
An important part of problems in statistical physics and computer
science can be expressed as the computation of marginal probabilities
over a Markov Random Field. The belief propagation algorithm, which is
an exact procedure to compute these marginals when the underlying
graph is a tree, has gained its popularity as an efficient way to
approximate them in the more general case. In this paper, we focus on
an aspect of the algorithm that did not get that much attention in the
literature, which is the effect of the normalization of the messages.
We show in particular that, for a large class of normalization
strategies, it is possible to focus only on belief convergence.
Following this, we express the necessary and sufficient conditions for
local stability of a fixed point in terms of the graph structure and
the beliefs values at the fixed point. We also explicit some connexion
between the normalization constants and the underlying Bethe Free
Energy.
\end{abstract}

\fi
\section{Introduction}

We are interested in this article in a random Markov field on a finite
graph with local interactions, on which we want to compute marginal
probabilities. The structure of the underlying model is described by a
set of discrete variables $\x=\{x_i,i\in \V\}\in\{1,\ldots,q\}^\V$,
where the set $\V$ of variables is linked together by so-called
``factors'' which are subsets $a\subset \V$ of variables. If $\F$ is
this set of factors, we consider the set of probability measures of
the form
\begin{equation}\label{eq:joint}
p(\x) = \prod_{i\in\V}\phi_i(x_i) \prod_{a\in\F}\psi_a(\x_a),
\end{equation}
where $\x_a=\{x_i, i\in a\}$.

$\F$ together with $\V$ define the factor graph $\G$~\citep{Kschi},
that is an undirected bipartite graph, which will be assumed to be
connected. We will also assume that the functions $\psi_a$ are never equal
to zero, which is to say that the Markov random field exhibits no
deterministic behavior. The set $\E$ of edges contains all the couples
$(a,i)\in\F\times\V$ such that $i\in a$. We denote $d_a$ (resp.\ $d_i$) the
degree of the factor node $a$ (resp.\ of the variable node $i$), and $C$
the number of independent cycles of $\G$.

Exact procedures to compute marginal probabilities of $p$ generally
face an exponential complexity problem and one has to resort to approximate
procedures. The Bethe approximation, which is used in
statistical physics, consists in minimizing an approximate version of the
variational free energy associated to~(\ref{eq:joint}). In computer
science, the belief propagation (BP) algorithm~\citep{Pearl} is a
message passing procedure that allows to compute efficiently exact
marginal probabilities when the underlying graph is a tree. When the
graph has cycles, it is still possible to apply the procedure, which
converges with a rather good accuracy on sufficiently sparse graphs.
However, there may be several fixed points, either stable or unstable.
It has been shown that these fixed points coincide with stationary points
of the Bethe free energy~\citep{YeFrWe3}. In addition~\citep{Heskes4,WaFu},
stable fixed points of BP are local minima of the Bethe free energy. We
will come back to this variational point of view of the BP algorithm in
Section~\ref{sec:var}.

We discuss in this paper an aspect of the algorithm that did not get
that much attention in the literature, which is the effect of the
normalization of the messages on the behavior of the algorithm.
Indeed, the justification for normalization is generally that it
``improves convergence''. Moreover, different authors use different
schemes, without really explaining what are the difference between
these definitions.

\medskip

The paper is organized as follows: the BP algorithm and its various
normalization strategies are defined in Section~\ref{sec:algorithm}.
Section~\ref{sec:mnormalization} deals with the effect of different
types of messages normalization on the existence of fixed points.
Section~\ref{sec:bdynamic} is dedicated to the dynamic of the
algorithm in terms of beliefs and cases where convergence of messages
is equivalent to convergence of beliefs; moreover, it is shown that
normalization does not change belief dynamic. In
Section~\ref{sec:stability}, we show that normalization is required
for convergence of the messages, and provide some sufficient
conditions. Finally, in Section~\ref{sec:var}, we tackle the issue of
normalization in the variational problem associated to Bethe
approximation. New research directions are proposed in
Section~\ref{sec:conclusion}.

%

\section{The belief propagation algorithm}\label{sec:algorithm}

The belief propagation algorithm~\citep{Pearl} is a message passing
procedure, which output is a set of estimated marginal probabilities,
the beliefs $b_a(\x_a)$ (including single nodes beliefs $b_i(x_i)$).
The idea is to factor the marginal probability at a given site as a
product of contributions coming from neighboring factor nodes, which
are the messages. With definition \eqref{eq:joint} of the joint
probability measure, the updates rules read:
\begin{align}
m_{a\to i}(x_i) &\gets
\sum_{\x_{a\setminus i}} \psi_a(\x_a)\prod_{j\in a\setminus i} 
n_{j\to a }(x_j), \label{urules}\\[0.2cm]
n_{i \to a}(x_i) &\egaldef \phi_i(x_i)\prod_{a'\ni i, a'\ne a}
m_{a'\to i}(x_i), \label{urulesn}
\end{align}
where the notation $\sum_{\x_s}$ should be understood as summing all
the variables $x_i$, $i\in s\subset \V$, from $1$ to $q$.  At any
point of the algorithm, one can compute the current beliefs as
\begin{align}
b_i(x_i) &\egaldef 
\frac{1}{Z_i(m)}\phi_i(x_i)\prod_{a\ni i} m_{a\to i}(x_i),
\label{belief1}\\[0.2cm]
b_a(\x_a) &\egaldef 
\frac{1}{Z_a(m)}\psi_a(\x_a)\prod_{i\in a} n_{i\to a}(x_i),
\label{belief2} 
\end{align}
where $Z_i(m)$ and $Z_a(m)$ are the normalization constants that
ensure that
\begin{equation}
\label{eq:normb}
 \sum_{x_i} b_i(x_i) = 1,\qquad\sum_{\x_a}b_a(\x_a)=1.
\end{equation}
These constants reduce to $1$ when $\G$ is a tree.

In practice, the messages are often normalized so that
\begin{equation}\label{eq:normalization}
 \sum_{x_i=1}^q m_{a\to i}(x_i)= 1.
\end{equation}

However, the possibilities of normalization are not limited to this
setting. Consider the mapping
\begin{equation}\label{eq:theta}
 \Theta_{ai,x_i}(m)\egaldef
\sum_{\x_{a\setminus i}} \psi_a(\x_a)
\prod_{j\in a\setminus i}\biggl[\phi_j(x_j)
\prod_{a'\ni j, a'\ne a}m_{a'\to j}(x_j)\biggr].
\end{equation}

A normalized version of BP is defined by the update rule
\begin{equation}\label{eq:normrule}
  \tilde m_{a\to i}(x_i)
\gets\frac{\Theta_{ai,x_i}(\tilde m)}
       {Z_{ai}(\tilde m)}.
\end{equation}
where $Z_{ai}(\tilde m)$ is a constant that depends on the messages and
which, in the case of~(\ref{eq:normalization}), reads
\begin{equation}\label{eq:Zmess}
Z^\mathrm{mess}_{ai}(\tilde m) \egaldef \sum_{x=1}^q\Theta_{ai,x}(\tilde
m).
\end{equation}

In the remaining of this paper, (\ref{urules},\ref{urulesn}) will be
referred to as ``plain BP'' algorithm, to differentiate it from the
``normalized BP'' of \eqref{eq:normrule}. 

Following~\citet{Wain}, it is worth noting that the plain message
update scheme can be rewritten as
\begin{equation}\label{urule2}
 m_{a\to i}(x_i)\gets \frac{Z_a(m)b_{i|a}(x_i)}{Z_i(m)b_i(x_i)}m_{a\to
i}(x_i),
\end{equation}
where we use the convenient shorthand notation
\[
 b_{i|a}(x_i)\egaldef\sum_{\x_{a\setminus i}}b_a(\x_a).
\]

This suggests a different type of normalization, used in particular
by~\citet{Heskes4}, namely
\begin{equation}\label{eq:Zbel}
 Z^\text{bel}_{ai}(\tilde m) = \frac{Z_a(\tilde m)}{Z_i(\tilde m)},
\end{equation}
which leads to the simple update rule
\begin{equation}
\label{eq:simple_uprule}
 \tilde m_{a\to i}(x_i)\gets \frac{b_{i|a}(x_i)}{b_i(x_i)}\tilde
 m_{a\to i}(x_i).
\end{equation}

The following lemma recapitulates some properties shared by all
normalization strategies at a fixed point.

\begin{lem}\label{lem:norm} 
Let $\tilde m$ be such that
\[
  \tilde m_{a\to i}(x_i)=\frac{\Theta_{ai,x_i}(\tilde
m)}{Z_{ai}(\tilde m)}.
\]
The associated normalization constants satisfy
\begin{equation}\label{eq:rnorm}
Z_{ai}(\tilde m) = \frac{Z_a(\tilde m)}{Z_i(\tilde m)},\qquad\forall ai\in
\E,
\end{equation}
and the following compatibility condition holds.
\begin{equation}
\sum_{\mathrm{\x}_{a\setminus i}} b_a(\mathrm{\x}_a) =
b_i(x_i).\label{eq:compat}
\end{equation}
In particular, when $Z_{ai}\equiv 1$ (no normalization), all the $Z_a$
and $Z_i$ are equal to some common constant $Z$.
\end{lem}
\begin{proof}
The normalized update rule~\eqref{eq:normrule}, together with
(\ref{belief1})--(\ref{belief2}), imply
\[
\sum_{\x_a\backslash x_i} b_a(\x_a) = \frac{Z_i Z_{ai}}{Z_a}b_i(x_i).
\]
By definition of $Z_a$ and $Z_i$, $b_a$ and $b_i$ are normalized
to $1$, so summing this relation w.r.t $x_i$ gives~\eqref{eq:rnorm}
and the equation above reduces to~(\ref{eq:compat}).  
\end{proof}
It is known~\citep{YeFrWe3} that the belief propagation algorithm is an
iterative way of solving a variational problem, namely it minimizes over $b$ 
the  Bethe free energy $F(b)$ associated with (\ref{eq:joint}). 
\begin{equation}
\label{eq:FBethe}
  F(b) \egaldef \sum_{a,\x_a} b_a(\x_a) \log \frac{b_a(\x_a)}{\psi_a(\x_a)}+ 
\sum_{i,x_i} b_i(x_i) \log
\frac{b_i(x_i)^{1-d_i}}{\phi_i(x_i)}.
\end{equation}

Writing the Lagrangian of the minimization of (\ref{eq:FBethe}) with $b$ 
subject to the constraints (\ref{eq:compat}) and (\ref{eq:normb}), one obtains
\begin{equation*}
\mathcal{L}(b,\lambda,\gamma) = F(b) + \sum_{\substack{i
,a\ni i\\ x_i}} \lambda_{ai}(x_i)\Bigl(b_i(x_i) - \sum_{\x_a/x_i}b_a(\x_a)
\Bigr) - \sum_{i} \gamma_i\Bigl(\sum_{x_i} b_i(x_i) -1\Bigr).
\end{equation*}
The minima are stationary points of $\mathcal{L}(b,\lambda,\gamma)$ which
correspond to
\begin{equation*}
\begin{cases}
b_a(\x_a) &= \DD\frac{\psi_a(\x_a)}{e} \prod_{j\in a} \prod_{b
\ni j, b \ne a} m_{b\to j}(x_j) ,\,\forall a \in \F \\[0.2cm] 
b_i(x_i) &= \DD \phi_i(x_i)\exp(\frac{1}{d_i-1}-\gamma_i) \prod_{b\ni
i}m_{a\to i}(x_i) ,\,\forall i \in \V
\end{cases}
\end{equation*} 
with the (invertible) parametrization
\begin{equation*}
\lambda_{ai}(x_i) =  \log \prod_{b \ni i, b \ne a}
m_{b\to i}(x_i),
\end{equation*}
Enforcing constraints (\ref{eq:compat}) yields the BP fixed points
equations with normalization terms $\gamma_i$. We will return to this
variational setting in Section~\ref{sec:var}.

\section{Normalization and existence of fixed points}\label{sec:mnormalization}
We discuss here an aspect of the algorithm that did not get that much
attention in the literature, which is the equivalence of the fixed
points of the normalized and plain BP flavors. 
  
It is not immediate to check that the normalized version of the
algorithm does not introduce new fixed points, that would therefore
not correspond to true stationary points of the Bethe free
energy. We show in Theorem~\ref{thm:equivalence} that
the sets of fixed points are equivalent, except possibly when the
graph $\G$ has one unique cycle.

\bigskip

As pointed out by~\cite{MooijKappen07}, many different sets of
messages can correspond to the same set of beliefs. The following
lemma shows that the set of messages leading to the same beliefs is
simply constructed through linear mappings.

\begin{lem}
\label{lem:bel_inv}
Two set of messages $m$ and $m'$ lead to the same beliefs if, and only
if, there is a set of strictly positive constants $c_{ai}$ such that
\[
  m'_{a\to i}(x_i) = c_{ai} m_{a\to i}(x_i).
\]
\end{lem}
\begin{proof}
The direct part of the lemma is trivial. Concerning the other part, we
have from (\ref{belief1}) and (\ref{belief2})
\[
 \frac{b_a(\x_a)Z_a(m)}{\psi_a(\x_a)} = \prod_{j \in a}
\prod_{b \ni j, b \ne a} m_{b\to j}(x_j)
\]
\[
 \frac{b_i(x_i)Z_i(m)}{\phi_i(x_i)} = \prod_{a \ni i} m_{a\to i}(x_i).
\]

Assume the two vectors of messages $m$ and $m'$ lead to the
same set of beliefs $b$ and write $m_{a\to i}(x_i) = c_{ai,x_i}\,
m'_{a\to i}(x_i)$. Then, from the relation on $b_i$, the vector $\>c$
satisfies
\begin{equation}
\label{eq:inv_bi}
\prod_{a \ni i} c_{ai,x_i}=\prod_{a \ni i} \frac{m_{a\to i}(x_i)}{m'_{a\to i}(x_i)}= \frac{Z_i(m)}{Z_i(m')}
\egaldef v_i.
\end{equation}

Moreover, we want to preserve the beliefs $b_a$. Using~\eqref{eq:inv_bi},
we have
\begin{equation}
\label{eq:va_def}
 \prod_{j \in a} \frac{m_{a\to j}(x_j)}{m'_{a\to j}(x_j)} = \prod_{j \in a}
c_{aj,x_j} = \frac{Z_a(m')}{Z_a(m)} \prod_{i \in a} v_i \egaldef v_a,
\end{equation}
Since $v_i$ (resp. $v_a$) does not depend on the choice of $x_i$ (resp.
$\x_a$), \eqref{eq:va_def} implies the independence of $c_{ai,x_i}$ with
respect to $x_i$. Indeed, if we compare two vectors $\x_a$ and $\x_a'$
such that,
for all $i \in a\setminus j$, $x'_i = x_i$, but $x'_j \ne x_j$, then
 $c_{aj,x_j} = c_{aj,x'_j}$, which concludes the proof.
\end{proof}

\subsection{From normalized BP to plain BP}

We show that in most cases the fixed points of a normalized BP algorithm
(no matter the normalization used) are associated with fixed points of
the plain BP algorithm. Recall that $C$ is the number of independent
cycles of $\G$. 

\begin{theo}\label{thm:equivalence}
A fixed point $\tilde m$ of the BP algorithm with normalized messages
corresponds to a fixed point of the plain BP algorithm associated to the
same beliefs iff one of the two following conditions is satisfied:
\begin{enumerate}
\item[(i)] the graph $\G$ has either no cycle or more than one ($C\ne 1$);
\item[(ii)] $C=1$, and the normalization constants of the associated
beliefs are such that
\begin{equation}\label{eq:prodZ}
\prod_{a\in\F}Z_a(\tilde m)\prod_{i\in\V}Z_i(\tilde m)^{1-d_i}=1.
\end{equation}
\end{enumerate}
\end{theo}
\begin{proof}
Let $\tilde m$ be a fixed point of~(\ref{eq:normrule}).
Let us find a set of constants $c_{ai}$ such that $m_{a\to i}(x_i)
=c_{ai}\,\tilde m_{a\to i}(x_i)$ is a non-zero fixed point
of~(\ref{urules},~\ref{urulesn}). Using Lemma~\ref{lem:bel_inv}, we see that
$m$ and $\tilde m$ correspond to the same beliefs. We have
\begin{align*}
 \Theta_{ai,x_i}(m) 
&= \biggl[\prod_{j\in a\setminus i}\prod_{a' \ni j, a'\ne a}c_{a'j}\biggr]
   \Theta_{ai,x_i}(\tilde m)\\
&= \biggl[\prod_{j\in a\setminus i}\prod_{a' \ni j, a'\ne a}c_{a'j}\biggr]
    Z_{ai}\,\tilde m_{a\to i}(x_i)\\
&= \frac{1}{c_{ai}}
   \biggl[\prod_{j\in a\setminus i}\prod_{a' \ni j, a'\ne a}c_{a'j}\biggr]
    Z_{ai}\,m_{a\to i}(x_i),
\end{align*}
and therefore
\[
\log c_{ai} 
 - \sum_{j\in a\setminus i}\sum_{a' \ni j, a'\ne a}\log c_{a'j}
=\log Z_{ai}.
\]
This equation is precisely in the setting of Lemma~\ref{lem:inverse} given in the Appendix,
with $x_{ai} = \log c_{ai}$ and $y_{ai} = \log Z_{ai}=\log Z_a - \log
Z_i$. It always has a solution when $C\ne 1$; when $C=1$, the
additional condition (\ref{eq:zeroF}) is required,
and~(\ref{eq:prodZ}) follows.
\end{proof}

There is in general an infinite number of fixed points $m$
corresponding to each $\tilde m$. However, as noted at the beginning
of the section, this is not a problem, since all these fixed points
correspond to the same set of beliefs. In this sense, normalizing the
messages can have the effect of collapsing equivalent fixed points.

When $C=1$, it is known~\citep{Weiss} that normalized BP always
converges to a fixed point. However, the theorem above states that
there may be no basic fixed point $m$ corresponding to a given $\tilde
m$. 

It is actually not difficult to see what happens in this case:
assume a trivial network with two variables and two factors $a=b=\{1,2\}$
and assume for simplicity that $\phi_1=\phi_2=1$. The
equations for the BP fixed point boil down to relations like
\[
 m_{a\to 1}(x_i)=\sum_{x_2}\psi_a(x_1,x_2)m_{b\to 2}(x_2),
\]
or, with a matrix notation,
\[
\>m_{a\to 1}=\Psi_a\>m_{b\to 2}=\Psi_a\Psi_b\>m_{a\to 1}.
\]

Therefore, the matrix $\Psi_a\Psi_b$ necessarily has $1$ as an
eigenvalue. Since this is not true in general, there can be no fixed point
for basic BP\@. In the normalized case,~\citet{Weiss} shows that BP
always converges to the Perron vector of this matrix. We know there is
an infinite number (not even countable, see Lemma~\ref{lem:bel_inv})
set of messages corresponding to the same beliefs. 

It is possible that the behavior of the algorithm leads to convergence
of the beliefs without the convergence of messages as the case $C=1$
suggests. Indeed, the plain BP scheme is then a linear dynamical
system which can converge to a subspace as described in \citet{Hart}.
We will describe more precisely this kind of behavior in
Section~\ref{sec:bdynamic}.

\subsection{From plain BP to normalized BP}

It turns out that there is no general result about whether a plain BP
fixed point is mapped to a fixed point by normalization. In this
section, we will thus first examine the case of a fairly general
family of normalizations, and then look at two other examples. 

\begin{defin}
A normalization $Z_{ai}$ is said to be \emph{positive homogeneous}
when it is of the form $Z_{ai} = N_{ai} \circ \Theta_{ai}$, with
$N_{ai}\,:\;\R^q\mapsto \R$ positive homogeneous functions of order $1$
satisfying
\begin{align}
\label{eq:Nlin}
 N_{ai}(\lambda m_{a\to i}) &=
\lambda N_{ai}(m_{a\to i}), \forall \lambda \geq 0.\\
\label{eq:Nai_pos}
N_{ai}(m_{a\to i}) &= 0 \iff  m_{a \to i} = 0.
\end{align}
\end{defin}
The part $\impliedby$ of \eqref{eq:Nai_pos} is obviously implied by
\eqref{eq:Nlin}. A particular family of positive homogeneous
normalizations is built from all norms $N_{ai}$ on $\R^q$. These contain in
particular the normalization $Z_{ai}^\mathrm{mess}(m)$
\eqref{eq:Zmess} or the maximum of messages
\[
 Z^\infty_{ai}(m)\egaldef\max_{x}\Theta_{ai,x}(m).
\]

It is actually not necessary to have a proper norm: 
\citet{WaFu} use a scheme that amounts to
\[
 Z^1_{ai}(m)\egaldef\Theta_{ai,1}(m).
\]

The following proposition describes the effect of the above family of
normalizations.
\begin{prop}
\label{prop:mess_norm_eq}
All the fixed points of the plain BP algorithm leading to the same set of
beliefs correspond to a unique fixed point of a positive homogeneous
normalized scheme.
\end{prop}
\begin{proof}
Let $m$ be a fixed point of the plain BP scheme. Using
Lemma~\ref{lem:bel_inv}, a fixed point $\tilde m$ of the normalized scheme
associated with the same beliefs than $m$ is such as
\begin{equation}\label{eq:linearm}
 \tilde m_{a\to i}(x_i) = c_{ai} m_{a\to i}(x_i).
\end{equation}
Since $\Theta$ is multi-linear,
\[
 \Theta_{ai,x_i}(\tilde m) = \left( \prod_{j \in a \setminus i} 
 \prod_{d \ni j, d \ne a} c_{dj} \right) \Theta_{ai,x_i}(m),
\]
and, using (\ref{eq:Nlin}),
\begin{align*}
 Z_{ai}(\tilde m) &= \left(
\prod_{j \in a \setminus i} \prod_{d \ni j, d \ne a} c_{dj} \right)
Z_{ai}(m),\\
 \tilde m_{a \to i}(x_i) &= \frac{\Theta_{ai,x_i}(\tilde m)}{Z_{ai}(\tilde
m)} = \frac{ m_{a \to i}(x_i)}{Z_{ai}(m)}.
\end{align*}
Therefore, $\tilde m$ is determined uniquely from $m$. Since $\tilde
m$ is clearly invariant for all the set of messages $m$ corresponding
to the same beliefs (see Lemma~\ref{lem:bel_inv}), the proof is
complete.
\end{proof}

\medskip
  
In order to emphasize the result of
Proposition~\ref{prop:mess_norm_eq}, it is interesting to describe
what happens with the belief normalization $Z^\text{bel}$
\eqref{eq:Zbel}. We know from Lemma~\ref{lem:norm} that, for any
normalization, we have at any fixed point
\[
Z_{ai}(m) = \frac{Z_a(m)}{Z_i(m)} \egaldef Z_{ai}^\text{bel}(m).
\]
Therefore, any fixed point of any normalized scheme (even of the plain
scheme) is a fixed point of the scheme with normalization $Z^\mathrm{bel}$.
We see the difference between this kind of normalization and a positive
homogeneous one. While the latter collapses families of fixed points to one
unique fixed point, $Z^\mathrm{bel}$ instead conserves all the fixed points
of all possible schemes.
  
\medskip

To conclude this section, we will present an example of a ``bad
normalization'' to illustrate a worst case scenario. Consider the
following normalization
\[
  Z_{ai}(m) = \frac{\sum_x \Theta_{ai,x}(m)}{\sup_x m_{a\to i}(x)}.
\]
This normalization, which is not homogeneous at all, defines a BP
algorithm which does not admit any fixed point. Following the proof of
Proposition~\ref{prop:mess_norm_eq}, let $\tilde m$ be a fixed point
of normalized BP associated with a plain fixed point $m$
through~\eqref{eq:linearm}, then
\[
 \tilde m_{a\to i}(x_i) = \frac{\Theta_{ai,x_i}(\tilde m)}{Z_{ai}(\tilde
m)} = \frac{\tilde
m_{a\to i}(x_i)}{Z_{ai}(m)}
\]
Indeed it is easy to check that 
\[
 Z_{ai}(\tilde m) = \frac{\prod_{j \in a \setminus i} \prod_{b \ni j, b\ne
a} c_{bj}}{c_{ai}} Z_{ai}(m).
\]
Since for any fixed point $m$ of the plain update we have
$Z_{ai}(m)>1$, no message $\tilde m$ can be a fixed point for this
normalized scheme. Using
Theorem~\ref{thm:equivalence} we conclude that this scheme admits no fixed
point.

\section{Belief dynamic}
\label{sec:bdynamic}

We are interested here in looking at the dynamic in terms of
convergence of beliefs. At each step of the algorithm, using
\eqref{belief1} and \eqref{belief2}, we can compute the current
beliefs $b\n_i$ and $b\n_a$ associated with the message $m\n$. The
sequence $m\n$ will be said to be ``$b$-convergent'' when the
sequences $b\n_i$ and $b\n_a$ converge. The term ``simple convergence''
will be used to refer to convergence of the sequence $m\n$ itself.
Simple convergence obviously implies $b$-convergence. We will first
show that for a positive homogeneous normalization, $b$-convergence
and simple convergence are equivalent. We will then conclude by
looking at $b$-convergence in a quotient space introduced in
\cite{MooijKappen07} and we show the links between these two
approaches.
\begin{prop}
 \label{prop:bconv_mconv}
 For any positive homogeneous normalization $Z_{ai}$ with
 \emph{continuous} $N_{ai}$, simple convergence and $b$-convergence
 are equivalent.
\end{prop}

\begin{proof}
Assume that the sequences of beliefs, indexed by iteration $n$, are such
that $b\n_a\to b_a$ and $b\n_i\to b_i$ as $n\to\infty$. The
idea of the proof is first to express the normalized messages $\tilde
m\n_{a\to i}$ at each step in terms of these beliefs, and then to conclude
by a continuity argument. Starting from a rewrite of
\eqref{belief1}--\eqref{belief2},
\begin{align*}
b_i\n(x_i) &= \frac{\phi_i(x_i)}{Z_i(\tilde m\n)}\prod_{a \ni i} \tilde m_{a\to i}\n(x_i), \\
 b_a\n(\x_a)  &=\frac{\psi_a(\x_a)}{Z_a(\tilde m\n)}\prod_{j \in a} \phi_j(x_j) \prod_{b \ni j, b
\ne a}
\tilde m_{b\to j}\n(x_j),
\end{align*}
one obtains by recombination
\[
 \prod_{j \in a} \tilde m\n_{a\to j}(x_j) 
 = \frac{\prod_{j \in a} Z_j(\tilde m\n)}{Z_a(\tilde m\n)} 
      \psi_a(\x_a) \frac{\prod_{j \in a} b_j\n(x_j)}{b_a\n(\x_a)} 
 \egaldef \frac{K\n_{ai}(\x_{a\setminus i};x_i)}{\tilde Z_{ai}(\tilde m)},
\]
where an arbitrary variable $i\in a$ has been singled out and
\[
 \frac{1}{\tilde Z_{ai}(\tilde m)}\egaldef\frac{\prod_{j \in a} Z_j(\tilde m\n)}{Z_a(\tilde m\n)}. 
\]

Assume now that $\x_{a\setminus i}$ is fixed and consider
$\>K\n_{ai}(\x_{a\setminus i})\egaldef K\n_{ai}(\x_{a\setminus
  i};\cdot)$ as a vector of $\R^q$. Normalizing each side of the
equation with a positive homogeneous function $N_{ai}$ yields
\[
 \frac{\tilde m\n_{a\to i}(x_i)}{N_{ai}\bigl[\tilde m\n_{a\to
i}\bigr]}
 = \frac{K\n_{ai}(\x_{a\setminus
i};x_i)}{N_{ai}\bigl[\>K\n_{ai}(\x_{a\setminus i})\bigr]}.
\]

Actually $N_{ai}\bigl[\tilde m\n_{a\to i}\bigr]=1$, since
$\tilde m\n_{a\to i}$ has been normalized by $N_{ai}$ and therefore
\[
 \tilde m\n_{a\to i}(x_i)
 = \frac{K\n_{ai}(\x_{a\setminus
i};x_i)}{N_{ai}\bigl[\>K\n_{ai}(\x_{a\setminus i})\bigr]}.
\]

This conclude the proof, since $\tilde m\n_{a\to i}$ has been expressed
as a continuous function of $b\n_i$ and $b\n_a$, and therefore it
converges whenever the beliefs converge.
\end{proof}

\bigskip
We follow now an idea developed in \cite{MooijKappen07} and study the
behavior of the BP algorithm in a quotient space corresponding to the
invariance of beliefs. First we will introduce a natural parametrization
for which the quotient space is just a vector space. Then it will be
trivial to show that, in terms of $b$-convergence, the effect of
normalization is null.

The idea of $b$-convergence is easier to express with the
new parametrization :
\begin{equation*}
 \mu_{ai}(x_i) \egaldef \log m_{a\to i}(x_i),
\end{equation*}
so that the plain update mapping (\ref{eq:theta}) becomes
\begin{equation*}
 \Lambda_{ai,x_i}(\mu) = \log \left[\sum_{\x_a \setminus i}
\psi_a(\x_a)\exp
\Bigl(\sum_{j \in a \setminus i} \sum_{\substack{b \ni j\\ b
\ne a}} \mu_{bj}(x_j)\Bigr)\right].
\end{equation*}
We have $\mu \in \N \egaldef \R^{|\E|q}$ and we define the vector
space $\W$ which is the linear span of the following vectors
$\{e_{ai} \in \N\}_{(ai)\in\E}$
\[
 (e_{ai})_{cj,x_j} \egaldef \1_{\{ai=cj\}}.
\]
It is trivial to see that the invariance set of the beliefs corresponding
to $\mu$ described in Lemma~\ref{lem:bel_inv} is simply the affine space
$\mu + \W$. So the $b$-convergence of a sequence $\mu^{(n)}$ is simply the
convergence of $\mu^{(n)}$ in the quotient space $\N \setminus \W$ (which
is a vector space, see \cite{Halmos}). Finally we define the notation $[x]$
for the canonical projection of $x$ on $\N \setminus \W$.

Suppose that we resolve to some kind of normalization on $\mu$, it is easy
to see that this normalization plays no role in the quotient space. The
normalization on $\mu$ leads to $\mu + w$ with some $w \in \W$. We have
\begin{align*}
 \Lambda_{ai,x_i}(\mu + w) &= \log \Bigl(\sum_{j \in a \setminus i}
\sum_{\substack{b \ni j\\b \ne a}} w_{bj} \Bigr) + \Lambda_{ai,x_i}(\mu)\\
 &\egaldef l_{ai} + \Lambda_{ai,x_i}(\mu),
\end{align*}
which can be summed up by 
\begin{equation}
\label{eq:quo}
 [\Lambda(\mu + \W)] = [\Lambda(\mu)],
\end{equation}
since $l\in \W$. We conclude by a proposition which is directly implied by
\eqref{eq:quo}.
\begin{prop}
\label{prop:dynamic}
The dynamic, i.e. the value of the normalized beliefs at each step, of the
BP algorithm with or without normalization is exactly the same.
\end{prop}
We will come back to this vision in term of quotient space in section
\ref{ssec:b2qs}.

\section{Local stability of BP fixed points}\label{sec:stability}

The question of convergence of BP has been addressed in a series of
works \citep{Tatikonda02,MooijKappen07,Ihler} which establish
conditions and bounds on the MRF coefficients for having global
convergence. In this section, we change the viewpoint and, instead of
looking for conditions ensuring a single fixed point, we examine the
different fixed points for a given joint probability and their local
properties.

In what follows, we are interested in the local stability of a message
fixed point $m$ with associated beliefs $b$. It is known that a BP
fixed point is locally attractive if the Jacobian of the relevant
mapping ($\Theta$ or its normalized version) at this point has a
spectral radius strictly smaller than $1$ and unstable when the
spectral radius is strictly greater than $1$. The term ``spectral
radius'' should be understood here as the modulus of the largest
eigenvalue of the Jacobian matrix.

We will first show that BP with plain messages can in fact
never converge when there is more than one cycle
(Theorem~\ref{thm:unstable}), and then explain how normalization of
messages improves the situation (Proposition~\ref{prop:Jtilde},
Theorem~\ref{thm:stability}).

\subsection{Unnormalized messages}\label{ssec:stabplain}

The characterization of the local stability relies on two ingredients.
The first one is the oriented line graph $L(\G)$ based on $\G$, which
vertices are the elements of $\E$, and which oriented links relate
$ai$ to $a'j$ if $j\in a\cap a'$, $j\ne i$ and $a'\ne a$. The
corresponding $0$-$1$ adjacency matrix $A$ is defined by the
coefficients
\begin{equation}\label{defA}
A_{ai}^{a'j} \egaldef\ind{j\in a\cap a',\,j\ne i,\, a'\ne a}.
\end{equation}

The second ingredient is the set of stochastic matrices
$B^{(iaj)}$, attached to pairs of variables $(i,j)$ having a
factor node $a$ in common, and which coefficients are the conditional
beliefs,
\begin{equation*}
b_{k\ell}^{(iaj)} \egaldef b_a(x_j=\ell\vert x_i = k)
= \sum_{\x_{a\setminus\{i,j\}}}\frac{b_a(\x_a)}{b_i(x_i)}
\Bigg|_{\substack{x_i=k\\ x_j=\ell}}
\end{equation*}
for all $(k,\ell)\in\{1,\ldots,q\}^2$. 

Using the representation~(\ref{urule2}) of the BP algorithm, the
Jacobian reads at this point:
\begin{align*}
\frac{\partial \Theta_{ai,x_i}(m)}{\partial m_{a'\to j}(x_j)}
&=\sum_{\x_{a\setminus \{i,j\}}} 
       \frac{b_a(\x_a)}{b_i(x_i)}
       \frac{m_{a\to i}(x_i)}{m_{a'\to j}(x_j)}
       \ind{j\in a\setminus i}\ind{a' \ni j, a'\ne a}\\
&= \frac{b_{ij|a}(x_i,x_j)}{b_i(x_i)}
   \frac{m_{a\to i}(x_i)}{m_{a'\to j}(x_j)}A_{ai}^{a'j}
\end{align*}

Therefore, the Jacobian of the plain BP algorithm is---using a trivial
change of variable---similar to the
matrix $J$ defined, for any pair $(ai,k)$ and $(a'j,\ell)$ of
$\E\times\{1,\ldots,q\}$ by the elements
\begin{equation*}
J_{ai,k}^{a'j,\ell}  \egaldef b_{k\ell}^{(iaj)} A_{ai}^{a'j},
\end{equation*}
This expression is analogous to the Jacobian encountered in
\citet{MooijKappen07}. It is interesting to note that it only depends
on the structure of the graph and on the belief corresponding to the
fixed point. 

Since $\G$ is a singly connected graph, it is clear that $A$ is an
irreducible matrix. To simplify the discussion, we assume in the
following that $J$ is also irreducible. This will be true as long as
the $\psi$ are always positive. It is easy to see that to any right
eigenvector of $A$ corresponds a right eigenvector of $J$ associated
to the same eigenvalue: if $\v=(v_{ai}, ai\in\E)$ is such that
$A\v=\lambda\v$, then the vector $\v^+$, defined by coordinates
$v^+_{a'j\ell}\egaldef v_{a'j}$, for all $a'j\in\E$ and
$\ell\in\{1,\ldots,q\}$, satisfies $J\v=\lambda\v$. We will say that
$\v^+$ is a $A$-based right eigenvector of $J$. Similarly, if $\u$ is
a left eigenvector of $A$, with obvious notations one can define a
$A$-based left eigenvector $\u^+$ of $J$ by the following coordinates:
$u^+_{aik}\egaldef u_{ai}b_i(k)$.

Using this correspondence between the two matrices, we can prove the
following result.

\begin{theo}\label{thm:unstable}
If the graph $\G$ has more than one cycle ($C>1$), and the matrix $J$
is irreducible, then the plain BP update
rules~(\ref{urules},~\ref{urulesn}) do not admit any stable fixed
point.
\end{theo}

\begin{proof}
Let $\boldsymbol\pi$ be the right Perron vector of $A$, which has
positive entries, since $A$ is irreducible~\cite[Theorem 1.5]{Sen}.
The $A$-based vector $\boldsymbol\pi^+$ also has positive coordinates
and is therefore the right Perron vector of $J$~\cite[Theorem
1.6]{Sen}; the spectral radius of $J$ is thus equal to the one of $A$.

When $C>1$, Lemma~\ref{lem:eigen} implies
that $1$ is an eigenvalue of $A$ associated to divergenceless vectors.
However, such vectors cannot be non-negative, and therefore the Perron
eigenvalue of $A$ is strictly greater than $1$. This concludes the
proof of the theorem.
\end{proof}

\subsection{Positively homogeneous normalization}

We have seen in Proposition~\ref{prop:bconv_mconv} that all the
continuous positively homogeneous normalizations make simple
convergence equivalent to $b$-convergence. As a result, one expects
that local stability of fixed points will again depend on the beliefs
structure only. Since all the positively homogeneous normalization
share the same properties, we look at the particular case of
$Z_{ai}^\text{mess}(m)$, which is both simple and differentiable. We
then obtain a Jacobian matrix with more interesting properties. In
particular, this matrix depends not only on the beliefs at the fixed
point, but also on the messages themselves: for the normalized BP
algorithm (\ref{eq:normrule} with $Z^\mathrm{mess}_{ai}$), the
coefficients of the Jacobian at fixed point $m$ with beliefs
$b$ read
\begin{align*}
\lefteqn{\frac{\partial}{\partial \tilde m_{a'\to j}(\ell)}
\biggl[\frac{\Theta_{ai,k}(\tilde m)}
       {\sum_{x=1}^q\Theta_{ai,x}(\tilde m)}\biggr]}\qquad\qquad\\
&= J_{ai,k}^{a'j,\ell}\frac{m_{a\to i}(k)}{m_{a'\to j}(\ell)}
   - m_{a\to i}(k)\sum_{x=1}^q J_{ai,x}^{a'j,\ell}
                               \frac{m_{a\to i}(x)}{m_{a'\to j}(\ell)},
\end{align*}
which is again similar to  the matrix $\widetilde J$ of general term
\begin{equation}\label{eq:jacobn}
\widetilde J_{ai,k}^{a'j,\ell}
\egaldef\biggl[b_{k\ell}^{(iaj)}- \sum_{x=1}^q m_{a\to i}(x) b_{x\ell}^{(iaj)}
\biggr] A_{ai}^{a'j}
= J_{ai,k}^{a'j,\ell} - \sum_{x=1}^q m_{a\to i}(x)J_{ai,x}^{a'j,\ell}.
\end{equation}

It is actually possible to prove that the spectrum of $\tilde J$ does
not depend on the messages themselves but only of the belief at the
fixed point.

\begin{prop}\label{prop:Jtilde}
The eigenvectors of $J$ are associated to eigenvectors of $\widetilde
J$ with the same eigenvalues, except the $A$-based eigenvectors of $J$
(including its Perron vector), which belong to the kernel of
$\widetilde J$.
\end{prop}
\begin{proof}
The new Jacobian matrix can be expressed from the old one as
$\widetilde J = (\mathbb{I}-M)J$, where $M$ is the matrix whose
coefficient at row $(ai,k)$ and column $(a'j,\ell)$ is $\ind{a=a',
  i=j}m_{a'\to j}(\ell)$. Elementary computations yield the following
properties of $M$:
\begin{itemize}
\item $M^2=M$: $M$ is a projector;
\item $\widetilde J M=0$.
\end{itemize}

For any right eigenvector $\v$ of $J$ associated to some eigenvalue
$\lambda$,
\begin{align*}
\widetilde J(\v-M\v)&=\widetilde J\v=(\mathbb{I}-M)J\v=\lambda(\v-M\v)
\end{align*}
so that $\v -M\v$ is a (right) eigenvector of $\widetilde J$ associated to
$\lambda$, unless $\v$ is an $A$-based eigenvector, in which case
$\v=M\v$ and $\v$ is in the kernel of $\widetilde J$.

Similarly, if $\u$ is such that $\u^T\widetilde
J=\lambda\u^T$ for $\lambda\ne 0$, then
$\lambda\u^TM=\u^T\widetilde J M=0$ and therefore
$\u^T\widetilde
J=\u^T(\mathbb{Id}-M)J=\u^TJ=\lambda\u^T$: any non-zero
eigenvalue of $\widetilde J$ is an eigenvalue of $J$. This proves the
last part of the theorem.
\end{proof}

As a consequence of this proposition, when $J$ is an irreducible matrix,
$\widetilde J$ has a strictly smaller spectral radius: the net effect
of normalization is to improve convergence (although it may actually
not be enough to guarantee convergence). To quantify this improvement
of convergence related to message normalization, we resort to
classical arguments used in speed convergence of Markov chains (see
e.g.~\citet{Bremaud}).

The presence of the messages in the Jacobian matrix $\widetilde J$
complicates the evaluation of this effect. However, it is known (see
e.g.~\citet{FuLaAu}) that it is possible to chose the functions $\hat\phi$
and $\hat\psi$ as
\begin{equation}\label{eq:phipsibethe}
  \hat\phi_i(x_i)  \egaldef \hat b_i(x_i), \qquad
  \hat\psi_a(\x_a) \egaldef\frac{\hat b_a(\x_a)}{\prod_{i\in a}\hat
b_i(x_i)},
\end{equation}
in order to obtain a prescribed set of beliefs $\hat b$ at a
fixed point. Indeed, BP will admit a fixed point
with $b_a=\hat b_a$ and $b_i=\hat b_i$ when $m_{a\to i}(x_i)\equiv 1$.
Since only the beliefs matter here, without loss of generality,
we restrict ourselves in the remainder of this section to the
functions~\eqref{eq:phipsibethe}. Then, from
\eqref{eq:jacobn}, the definition of $\widetilde J$ rewrites
\begin{equation*}
\widetilde J_{ai,k}^{a'j,\ell}
\egaldef\biggl[b_{k\ell}^{(iaj)}- \frac{1}{q}\sum_{x=1}^q b_{x\ell}^{(iaj)}
\biggr] A_{ai}^{a'j}
= J_{ai,k}^{a'j,\ell} - \frac{1}{q}\sum_{x=1}^q J_{ai,x}^{a'j,\ell}.
\end{equation*}

For each connected pair $(i,j)$ of variable nodes, we associate to the
stochastic kernel $B^{(iaj)}$ a combined stochastic kernel
$K^{(iaj)}\egaldef B^{(iaj)}B^{(jai)}$, with coefficients
\begin{equation}\label{def:kernel}
K_{k\ell}^{(iaj)} 
\egaldef \sum_{m=1}^q b_{km}^{(iaj)} b_{m\ell}^{(jai)}.
\end{equation}
Since $b^{(i)}B^{(iaj)}=b^{(j)}$, $b^{(i)}$ is the
invariant measure associated to $K$:
\[
 b^{(i)}K^{(iaj)}=b^{(i)}B^{(iaj)}B^{(jai)} 
 = b^{(j)}B^{(jai)} = b^{(i)}
\]
and $K^{(iaj)}$ is reversible, since
\begin{align*}
b_k^{(i)}  K_{k\ell}^{(iaj)} &= \sum_{m=1}^q b_{mk}^{(jai)} 
b_m^{(j)} b_{m\ell}^{(jai)}\\[0.2cm] 
&= \sum_{m=1}^q b_{mk}^{(jai)} b_{\ell m}^{(iaj)} b_\ell^{(i)} 
= b_\ell^{(i)} K_{\ell k}^{(iaj)}.
\end{align*}

Let $\mu_2^{(iaj)}$ be the second largest eigenvalue of $K^{(iaj)}$ and
let
\begin{equation*}
\mu_2 \egaldef \max_{ij} |\mu_2^{(iaj)}|^{\frac{1}{2}}.
\end{equation*}

The combined effect of the graph and of the local correlations, 
on the stability of the reference fixed point is stated as follows.
\begin{theo}\label{thm:stability}
Let $\lambda_1$ be the Perron eigenvalue of the matrix $A$
\begin{enumerate}
\item[(i)] if $\lambda_1\mu_2 < 1$, the fixed point of the normalized BP
schema (\ref{eq:normrule} with $Z^\mathrm{mess}_{ai}$)
associated to $b$ is stable.
\item[(ii)] condition (i) is necessary and sufficient if the system is
homogeneous ($B^{(iaj)} = B$ independent of $i$, $j$ and $a$),
with $\mu_2$ representing the second largest eigenvalue of $B$.
\end{enumerate}
\end{theo}

\begin{proof}
See Appendix~\ref{app:stability}
\end{proof}

The quantity $\mu_2$ is representative of the level of mutual
information between variables. It relates to the spectral gap (see
e.g.~\citet{DiaStr} for geometric bounds) of each elementary stochastic
matrix $B^{(iaj)}$, while $\lambda_1$ encodes the statistical
properties of the graph connectivity. The bound $\lambda_1\mu_2 <1$
could be refined when dealing with the statistical average of the sum
over path in (\ref{eq:paths}) which allows to define $\mu_2$ as
\[
\mu_2 = \lim_{n\to\infty}\max_{(ai,a'j)}
\Bigl\{\frac{1}{|\Gamma_{ai,a'j}\n|}
\sum_{\gamma\in\Gamma_{ai,a'j}\n}
\Bigl(\prod_{(x,y)\in\gamma}\mu_2^{(xy)}\Bigr)^{\frac{1}{2n}}\Bigr\}.
\]

\subsection{Local convergence in quotient space $\N \setminus \W$}
\label{ssec:b2qs}

The idea is to make the connexion between local stability of fixed point as
described previously and the same notion of local stability but in the
quotient space $\N \setminus \W$ described in Section~\ref{sec:bdynamic}.
Trivial computation based on the results of Section~\ref{ssec:stabplain}
gives us the derivatives of $\Lambda$.
\[
 \frac{\partial \Lambda_{ai,x_i}(\mu)}{\partial \mu_{bj}(x_j)}
=
\frac{b_{ij|a}(x_i,x_j)}{b_i(x_i)} A_{ai}^{bj} = J_{ai,x_i}^{bj,x_j}.
\]
In terms of convergence in $\N\setminus\W$, the stability of a fixed point
is given by the projection of $J$ on the quotient space $\N \setminus \W$ 
and we have \citep{MooijKappen07} :
\[
[J] \egaldef [\nabla \Lambda] = \nabla [\Lambda]
\]
\begin{prop}
The eigenvalues of $[J]$ are the eigenvalues of $J$ which are not
associated with $A$-based eigenvectors. The $A$-based eigenvectors of $J$
belong to the kernel of $[J]$
\end{prop}
\begin{proof}
Let $v$ be an eigenvector of $J$ for the eigenvalue $\lambda$, we have
\[
 [Jv] = [\lambda v] = \lambda [v], 
\]
so $[v]$ is an eigenvector of $[J]$ with the same
eigenvalue $\lambda$ iff $[v]\ne 0$.
The $A$-based eigenvectors (see Section~\ref{ssec:stabplain})
$w$ of $J$ belongs to $\W$ so we have 
\[
 [w] = 0.
\]
It means that these eigenvectors of $J$ have no equivalent w.r.t $[J]$
and play no role in belief fixed point stability.
\end{proof}

We have seen that the normalization $Z^\mathrm{mess}_{ai}$ is equivalent to
multiplying the jacobian matrix $J$ by the projection $\I-M$ (Proposition
\ref{prop:Jtilde}), with
\[
\ker (\I-M) = \W.
\]
The projection $\I-M$ is in fact a quotient map from $\N$ to $\N\setminus
\W$. So the normalization $Z^\mathrm{mess}_{ai}$ is strictly equivalent,
when we look at the messages $m_{a\to i}(x_i)$, to working on the quotient
space $\N \setminus \W$. More generally for any differentiable positively
homogeneous normalization we will obtain the same result, the jacobian of
the corresponding normalized scheme will be the projection of the jacobian
$J$ on the quotient space $\N\setminus\W$, through some quotient map.

\section{Normalization in the variational problem}
\label{sec:var}

Since Proposition~\ref{prop:dynamic} shows that the choice of
normalization has no real effect on the dynamic of BP, it will have no
effect on $b$-convergence either. In this section, we turn to the
effect of normalization on the underlying variational problem. It will
be assumed here that the beliefs $b_i$ and $b_a$ are normalized
(\ref{eq:normb}) and compatible (\ref{eq:compat}). If only
(\ref{eq:compat}) is satisfied, they will be denoted  $\beta_i$
and $\beta_a$.
It is quite obvious that imposing only compatibility constraints leads
to a unique normalization constant $Z$
\[
 Z(\beta) \egaldef \sum_{x_i} \beta_i(x_i) = \sum_{\x_a} \beta_a(\x_a),
\]
which is not \emph{a priori} related to the constants $Z_a(m)$ and
$Z_i(m)$ seen in the previous sections. The quantities
$\beta_i(x_i)/Z(\beta)$ and $\beta_a(\x_a)/Z(\beta)$ can be
denoted as $b_i(x_i)$ and $b_a(\>x_a)$ since (\ref{eq:normb}) holds for
them. 

The aim of this section is to explicit the relationship between the
minimizations of the Bethe free energy (\ref{eq:FBethe}) with and
without normalization constraints (\ref{eq:normb}). Generally
speaking, we can express them as a  minimization problem
$\P(E)$ on some set $E$ as
\begin{equation}
\label{eq:minpb}
\P(E)\quad:\quad \underset{\beta \in E}{\operatorname{argmin}}\,F(\beta) 
\end{equation}
where $E$ is chosen as follows
\begin{itemize}
\item plain case: $E=E_1$ is the set of positive measures such as 
(\ref{eq:compat}) holds,
\item normalized case: $E=E_2 \varsubsetneq E_1$ has the additional
constraint~(\ref{eq:normb}). 
\end{itemize}

It is possible to derive a BP algorithm for the plain problem
following the same path as in Section~\ref{sec:algorithm}. The
resulting update equations will be identical, except for the
$\gamma_i$ terms. 

The first step is to compare the solutions of \eqref{eq:minpb} on
$E_1$ and $E_2$. Let $\varphi$ be the bijection between $E_1$ and
$E_2\times\mathbb{R}^*_+ $,
\begin{align*}
 \varphi :  &\, E_2\times\mathbb{R}^*_+ \longrightarrow E_1 \\
           & (b,Z) \longrightarrow bZ.
\end{align*}
The variational problem $\P(E_1)$ is equivalent to
\begin{equation*}
(\hat b, \hat Z) =  \underset{(b,Z) \in E_2}{\operatorname{argmin}}\,
F(\varphi(b,Z)),
\end{equation*}
with $\varphi(\hat b, \hat Z) = \hat b \hat Z = \hat \beta \egaldef
\underset{\beta\in E_1}{\operatorname{argmin}}\, F(\beta)$.

The next step is to express the Bethe free energy $F(\beta)$ of an
unnormalized positive measure $\beta$ as a function of the Bethe free
energy of the corresponding normalized measure $b$.
\begin{lem}
\label{lem:Fzb}
As soon as the factor graph is connected, for any $\beta = Z b \in E_1$
we have
\begin{equation}
\label{eq:Fbeta}
 F(Zb) = Z \bigl(F(b) + (1-C) \log Z\bigr),
\end{equation}
with $C$ being the number of independent cycles of the graph.
\end{lem}
\begin{proof}
\begin{align*}
 F(\beta) &= F(Z b)\\
&= Z \Bigl[\sum_{a,\x_a} b_a(\x_a)\log\Bigl(\frac{Z
b_a(\x_a)}{\psi_a(\x_a)}\Bigr) + \sum_{i, x_i} 
b_i(x_i)\log \Bigl(\frac{\left(Z b_i(x_i)\right)^{1-d_i}}
{\phi_i(x_i)}\Bigr)\Bigr] \\
&= Z \Bigl( F(b) + (|\F| + |\V| - |\E| )\log Z \Bigr)\\
&= Z \Bigl( F(b) + (1-C)\log Z \Bigr),
\end{align*}
where the last equality comes from elementary graph theory (see e.g.
\cite{Berge}).
\end{proof}

The quantity $1-C$ will be negative in the nontrivial cases (at least
 $2$ cycles). Since all the $Z b$ are equivalent from our point of view, we
look at the derivatives of $F(Zb)$ as a function of $Z$ to see what happens
in the plain variational problem. 
\begin{theo}
\label{thm:eqpbvar}
The normalized beliefs corresponding to the extrema of the plain
variational problem $\P(E_1)$  are exactly the same as the ones of
the normalized problem $\P(E_2)$ as soon as $C \ne 1$.
\end{theo}
\begin{proof}
Using Lemma~\ref{lem:Fzb} we obtain
\[
 \frac{\partial F(\beta)}{\partial Z} = F(b) + (1-C) ( \log Z + 1),
\]
the stationary points are
\begin{equation}
\label{eq:relzf}
 \hat Z = \exp \Big(\frac{F(b)}{C-1} -1\Big).
\end{equation}
At these points we can compute the Bethe free energy
\[
 F(\hat \beta) = F(\hat Z b) 
 = (C-1) \exp\Big(\frac{F(b)}{C-1} -1 \Big)
 = G(F(b)).
\]
It is easy to check that, if $C \ne 1$, $G$ is an increasing
function, so the extrema of $F(\beta)$ are reached at the same
normalized beliefs. More precisely, if $b_1$ and $b_2$ are elements of
$E_2$ such that $F(b_1) \leq F(b_2)$ then  $F(\hat\beta_1 = \hat
Z_1 b_1) \leq F(\hat \beta_2 = \hat Z_2 b_2)$, which allows us to
conclude.
\end{proof}
In other words, imposing a normalization in the variational
problem or normalizing after a solution is reached is equivalent as long as
$C \ne 1$. Moreover, in the unnormalized case, the Bethe free energy at the local extremum writes
\begin{equation}
\label{eq:hatz}
 F(b) = (\mathcal{C}-1) (\log \hat Z + 1).
\end{equation}
We can therefore compare the ``quality'' of different fixed points by
comparing only the normalization constant obtained: the smaller $Z$
is, the better the approximation, modulo the fact that we're not
minimizing a true distance.

When $C = 1$, it has been shown already in
Section~\ref{sec:mnormalization} that the normalized scheme is always
convergent, whereas the plain scheme can have no fixed point.
In this case, (\ref{eq:Fbeta}) rewrites
\[
 F(\beta) = F(Zb) = Z F(b).
\]
The form of this relationship shows what happens: if the extremum of
the normalized variational problem is strictly negative, $F(\beta)$ is
unbounded from below and $Z$ will diverge to $+\infty$; conversely, if
the extremum is strictly positive, $Z$ will go to zero. In the (very)
particular case where the minimum of the normalized problem is equal
to zero, the problem is still well defined. In fact this condition
$F(b) = 0$ is equivalent to the one of Theorem~\ref{thm:equivalence}
when $\mathcal{C}=1$.

\bigskip

To sum up, as soon as the plain variational problem is well
defined, it is equivalent to the normalized one and the normalization constant
allows to compute easily the Bethe free energy using (\ref{eq:hatz}).
When this is no longer the case, we still know that the dynamics of both 
algorithms remain the same (Proposition~\ref{prop:dynamic}) but the
plain variational problem (which can still converge in terms of beliefs)
will not converge in terms of normalization constant $Z$, and we have no
more easy information on the fixed point free energy.


\bigskip

As emphasized previously, the relationship between $\hat Z$, $Z_a(m)$
and $Z_i(m)$ is not trivial. In the case of the plain BP algorithm,
for which $Z_a(m) = Z_i(m)$, an elementary computation yields the
following relation at any fixed point
\[
 F(b) = (\mathcal{C}-1) \log Z_a(m),
\]
which seemingly contradicts (\ref{eq:hatz}). In fact, the algorithm
derived from the plain variational problem is not exactly the plain BP
scheme. Usually, since one resorts to some kind of normalization, the
multiplicative constants of the fixed
point equations are discarded (see~\citet{YeFrWe3}
for more details). Keeping track of them  yields
\begin{equation}
\label{eq:TBP}
 m_{a \to i}(x_i) = \exp\left(\frac{d_i-2}{d_i-1} \right)
\Theta_{ai,x_i}(m),
\end{equation}

\begin{align*}
 \beta_a(\x_a) &= \frac{1}{e} \psi_a(\x_a) \prod_{j \in a}
n_{j\to a}(x_j), \\
 \beta_i(x_i) &= \phi_i(x_i) \exp\left(\frac{1}{d_i-1}\right) \prod_{b
\ni i}
m_{b\to i}(x_i).
\end{align*}
Actually, the plain update scheme (\ref{urules},\ref{urulesn})
corresponds to some constant normalization
$\exp\left(\frac{d_i-2}{d_i-1} \right)$. Without any normalization, using
\eqref{eq:TBP} as update rule, one would obtain
\[
 \hat Z = \frac{Z_a(m)}{e} = Z_i(m) \exp\left(\frac{1}{d_i-1}\right).
\]

\section{Conclusion}
\label{sec:conclusion}
This paper motivation was to fill a void in the literature about the
effect of normalization on the BP algorithm. What we have learnt can
be summarized in a few main points
\begin{itemize}
\item using a normalization in BP can in some rare cases kill or
create new fixed points;
\item not all normalizations are created equal when it comes to
message convergence, but there is a big category of positive
homogeneous normalization that all have the same effect;
\item the user is ultimately concerned with convergence of beliefs,
and thankfully the dynamic of normalized beliefs is insensitive to
normalization. 
\end{itemize}

The messages having no interest by themselves, it is worthy of remark
that combining the update rules \eqref{eq:simple_uprule} recalled below
\[
 m_{a\to i}(x_i) \gets \frac{b_{i|a}(x_i)}{b_i(x_i)} m_{a\to i}(x_i),
\]
and the definition \eqref{belief1} and \eqref{belief2} of beliefs, one
can eliminate the messages and obtain
\begin{align*}
b_i(x_i)&\gets b_i(x_i)\prod_{a\ni
i}\frac{b_{i|a}(x_i)}{b_i(x_i)},\\
b_a(\x_a)&\gets b_a(\x_a)\prod_{i\in a}\prod_{c\ni i, c\ne
a}\frac{b_{i|c}(x_i)}{b_i(x_i)},
\end{align*}
One particularity of these update rules is that they do not depend on the
functions $\psi$ or $\phi$ but only on the graph structure. The
dependency on the joint law \eqref{eq:joint} occurs only through the initial
conditions. This ``product sum'' algorithm  therefore shares common
properties for all models build on the same underlying graph,
and the initial conditions should impose the details of the joint law. To
our knowledge this algorithm has never been studied and we let it for
future work.

\appendix
\section{Spectral properties of the factor
  graph}\label{app:factorgraph}

This appendix is devoted to some properties of the matrix $A$ defined
in (\ref{defA}) that are used in Sections~\ref{sec:mnormalization} and~\ref{sec:stability}.

We consider two types of fields associated to $\G$, namely scalar
fields and vector fields. Scalar fields are quantities attached to the
vertices of the graph, while vector fields are attached to its edges.
A vector field $\>w=\{w_{ai},\ ai\in\E\}$ is \emph{divergenceless} if
\[
\forall a\in\F,\ \sum_{i\in a} w_{ai} = 0 \quad\text{and}\quad
\forall i\in\V,\ \sum_{a\ni i} w_{ai} = 0.
\]
A vector field $\>u=\{u_{ai},\ ai\in\E\}$ is
a \emph{gradient} if there exists a scalar field $\{u_a,u_i,\
a\in\F,\ i\in\V\}$ such that
\[
\forall ai\in\E,\ u_{ai} = u_a - u_i.
\]

There is an orthogonal decomposition of any vector field into a
divergenceless and a gradient component. Indeed, the scalar product
\[
 \>w^T \>u 
    = \sum_{ai\in\E} w_{ai}u_{ai}
    = \sum_{a\in\F} u_a\sum_{i\in a}w_{ai}
      - \sum_{i\in\V} u_i\sum_{a\ni i}w_{ai},
\]
is $0$ for all gradient fields $\>u$ iff $\>w$ is divergenceless.
Dimensional considerations show that any vector field $\>v$
can be decomposed in this way.

In the following, it will be useful to define the Laplace operator
$\Delta$ associated to $\G$. For any scalar field $\>u$:
\begin{align}
(\Delta \>u)_a \egaldef d_a u_a -\sum_{i\in a}u_i,\qquad\forall
a\in\F\label{def:laplace1}\\[0.2cm]
(\Delta \>u)_i \egaldef d_i u_i -\sum_{a\ni i}u_a,\qquad\forall
i\in\V.\label{def:laplace2}
\end{align}

The following lemma describes the spectrum of $A$ in terms of a
Laplace equation on the graph $\G$.
\begin{lem}\label{lem:eigen} 
(i) Both gradient and divergenceless vector spaces are $A$-invariant
and divergenceless vectors are eigenvectors of $A$ with eigenvalue $1$.
(ii) eigenvectors associated to eigenvalues $\lambda\ne1$ are gradient
vectors of a scalar field $\>u$ which satisfies
\begin{equation}\label{eq:valp}
\bigl(\Delta \>u \bigr)_a = \frac{(\lambda -1)(d_a-1)}{\lambda}u_a\
\text{and}\ 
\bigl(\Delta \>u \bigr)_i = (1-\lambda)u_i.
\end{equation}
and there exists a gradient vector associated to $1$ iff  $\G$ has exactly
one cycle ($C=1$).
\end{lem}
\begin{proof}
The action of $A$  on a given vector $\x$ reads 
\begin{equation*}
\sum_{a'j\in\E} A_{ai}^{a'j}x_{a'j}
= \sum_{j\in a}\Bigl(\sum_{a'\ni j} x_{a'j} - x_{aj}\Bigr)
-\sum_{a'\ni i}x_{a'i} +x_{ai},
\end{equation*}
The first two terms in the second member vanish if $\x$ is
divergenceless. In addition, the first term in parentheses is
independent of $i$ while the second one is independent of $a$ so the
first assertion is justified. We concentrate then on solving the
eigenvalue equation $A\x -\lambda \x =0$ for a gradient vector
$\x$, with $x_{ai} = u_a -u_i$.  $A\x -\lambda \x$ is the gradient
of a constant scalar $K\in\mathbb R$, and by identification we have
\[
\begin{cases}
\DD \bigl(\Delta \>u\bigr)_a + 
\sum_{j\in a}\bigl(\Delta \>u\bigr)_j = (1-\lambda)u_a +K\\[0.2cm]
\DD \bigl(\Delta \>u\bigr)_i = (1-\lambda)u_i +K.
\end{cases}
\]
The Laplacian of a constant scalar is zero, so for $\lambda\ne 1$, $K$
may be reabsorbed in $\>u$ and, combining these two equations with the
help of identities (\ref{def:laplace1},\ref{def:laplace2}), yields
equation (\ref{eq:valp}). For $\lambda=1$, we obtain
\begin{equation}\label{eq:valp2}
\bigl(\Delta \>u \bigr)_a = (1-d_a)K\qquad\text{and}\qquad 
\bigl(\Delta \>u \bigr)_i = K.
\end{equation}
Let $D$ be the diagonal matrix associated to the graph $\G$, whose
diagonal entries are the degrees $d_a$ and $d_i$ of each node. $M = \I
- D^{-1}\Delta$ is a stochastic irreducible matrix, which unique right
Perron vector $(1,\ldots,1)$ generates the kernel of $\Delta$. As a
result, for $K=0$, the solution to (\ref{eq:valp2}) is $u_a=u_i=cte$
so that $x_{ai}=0$. 

For $K\ne 0$, there is a solution if the second member of
(\ref{eq:valp2}) is orthogonal ($\Delta$ is a symmetric operator) to
the kernel. The condition reads
\[
0 = \sum_{a}(1-d_a) +\sum_i 1= |\F| - |\E| + |\V| = 1-C,
\]
where the last equality comes from elementary graph theory (see e.g.
\cite{Berge}).
\end{proof}

Since $1$ is an eigenvalue of $A$, it is interesting to investigate
linear equations involving $\I - A$. Since it is already known that
divergenceless vectors are in the kernel of this matrix, we
restrict ourselves to the case where the constant term is of gradient
type.

\begin{lem}\label{lem:inverse}
For a given gradient vector field $\y$, the equation
\begin{equation*}
\bigl(\I - A\bigr) \x = \y,
\end{equation*}
has a solution (unique up to a divergenceless vector) iff $C\ne 1$ or
$C=1$ and
\begin{equation}\label{eq:zeroF}
\sum_{a\in\F}y_a+ \sum_{i\in\V}(1-d_i) y_i = 0.
\end{equation}
\end{lem}
\begin{proof}
We look here only for gradient-type solutions $x_{ai} =
u_a - u_i$ and write $y_{ai} = y_a - y_i$. Owing to the same arguments as
in Lemma~\ref{lem:eigen}, there exists a constant $K$ such that
\begin{align*}
\bigl(\Delta \>u\bigr)_a  &= K(d_a-1) +y_a -\sum_{j\in a}y_j\\
\bigl(\Delta \>u\bigr)_i &= y_i -K.
\end{align*}
Stating as before the compatibility condition for this equation yields
\[
\sum_{a\in\F}y_a+ \sum_{i\in\V}(1-d_i) y_i = K(C-1).
\]
It is always possible to find a suitable $K$ as long as $C\ne 1$ and when
$C=1$,
(\ref{eq:zeroF}) has to hold.
\end{proof}

\section{Proof of Theorem~\ref{thm:stability}}\label{app:stability}

Let us start with (ii): when the system is
homogeneous, $\widetilde J$ is a tensor product of $A$ with
$\widetilde B$, and its spectrum is therefore the product of their
respective spectra. In particular if $\G$ has uniform degrees $d_a$
and $d_i$, the condition reads
\[
\mu_2 (d_a-1)(d_i-1) < 1.
\] 

In order to prove part (i) of the theorem, we will consider a local norm on
$\mathbb{R}^q$ attached to each variable node $i$,
\[
\|x\|_{b^{(i)}} \egaldef \Bigl(\sum_{k=1}^q
x_k^2 b_k^{(i)}\Bigr)^{\frac{1}{2}}\ 
\text{and}\ 
\langle x\rangle_{b^{(i)}} \egaldef \sum_{k=1}^q x_k b_k^{(i)},
\]
the local average of $x\in\mathbb{R}^q$ w.r.t $b^{(i)}$. For
convenience we will also consider the somewhat hybrid global norm on
$\mathbb{R}^{q\times|\E|}$
\[
\|x\|_{\pi,b} \egaldef \sum_{a\to i}\pi_{ai}\|x_{ai}\|_{b^{(i)}},
\]
where $\boldsymbol\pi$ is again the right Perron vector of $A$,
associated to $\lambda_1$.

We have the following useful inequality.
\begin{lem}\label{lem:ineq}
For any $(x_i,x_j)\in \mathbb{R}^{2q}$, such that $\langle
x_i\rangle_{b^{(i)}} = 0$ and $x_{j,\ell}b_\ell^{(j)} = \sum_k x_{i,k}
b_k^{(i)}B_{k\ell}^{(iaj)}$,
\[
\langle x_j\rangle_{b^{(j)}} = 0
\qquad \text{and}\qquad 
\|x_j\|_{b^{(j)}}^2 \le \mu_2^{(iaj)} \|x_i\|_{b^{(i)}}^2.
\]
\end{lem}
\begin{proof}
By definition (\ref{def:kernel}), we have
\begin{align*}
\|x^{(j)}\|_{b^{(j)}}^2 &= \sum_{k=1}^q \frac{1}{b_k^{(j)}} 
\Bigl|\sum_{\ell=1}^q b_{\ell
k}^{(iaj)}b_\ell^{(i)}x_\ell^{(i)}\Bigr|^2\\[0.2cm]
&= \sum_{\ell,m} x_\ell^{(i)}x_m^{(i)} K_{\ell m}^{(iaj)}b_\ell^{(i)}.
\end{align*}
Since $K^{(iaj)}$ is reversible we have from Rayleigh's theorem
\[
\mu_2^{(iaj)} \egaldef \sup_{x}
\Bigl\{\frac{\sum_{k\ell}x_k x_\ell K_{k\ell}^{(iaj)}b_k^{(i)}}{\sum_k
x_k^2 b_k^{(i)}},
\langle x\rangle_{b^{(i)}} = 0,x\ne 0\Bigr\},
\]
which concludes the proof.
\end{proof}
To deal with iterations of $J$, we express it as a sum over paths.
\[
\bigl(J^n\bigr)_{ai,k}^{a'j,\ell} = \bigr(A^n\bigr)_{ai}^{a'j} 
\bigl(B_{ai,a'j}\n\bigr)_{k\ell},
\]
where $B_{ai,a'j}\n$ is an average stochastic kernel,
\begin{equation}\label{eq:paths}
B_{ai,a'j}\n \egaldef \frac{1}{|\Gamma_{ai,a'j}\n|}
\sum_{\gamma\in\Gamma_{ai,a'j}\n}\prod_{(x,y)\in\gamma}B^{(xy)}. 
\end{equation}
$\Gamma_{ai,a'j}\n$ represents the set of directed path of length
$n$ joining $ai$ and $a'j$ on $L(\G)$ and its cardinal is precisely
$|\Gamma_{ai,a'j}\n| = \bigr(A^n\bigr)_{ai}^{a'j}$.
\begin{lem}\label{lem:ineq2}
For any $(x_{ai},x_{a'j})\in \mathbb{R}^{2q}$, such that $\langle
x_i\rangle_{b^{(i)}} = 0$ and
\[
 x_{a'j,\ell}b_\ell^{(j)} 
   = \sum_k x_{ai,k}b_k^{(i)}\bigl(B_{ai,a'j}\n\bigr)_{k\ell},
\] 
the following inequality holds
\[
\|x_{a'j}\|_{b^{(j)}} \le \mu_2^n \|x_{ai}\|_{b^{(i)}}.
\]
\end{lem}
\begin{proof}
Let $x_{a'j}^\gamma$ the contribution to $x_{a'j}$ corresponding to the
path $\gamma\in\Gamma_{ai,a'j}\n$. Using Lemma~\ref{lem:ineq}
recursively yields for each individual path
\[
\|x_{a'j}^\gamma\|_{b^{(j)}} \le \mu_2^n \|x_{ai}\|_{b^{(i)}},
\]
and, owing to triangle inequality,
\[
\|x_{a'j}\|_{b^{(j)}} \le
\frac{1}{|\Gamma_{ai,a'j}\n|}\sum_{\gamma\in\Gamma_{ai,a'j}\n}
\|x_{a'j}^\gamma\|_{b^{(j)}} \le \mu_2^n \|x_{ai}\|_{b^{(i)}}.
\]
\end{proof}

It is now possible to conclude the proof of the theorem.

\begin{proof}[Proof of Theorem~\ref{thm:stability}(i)]
(i) Let $\v$ and $\v'$ two vectors with $\v' = \v \tilde J^n =
\v(\I-M)J^n$, ($M$ is the projector defined in
Proposition~\ref{prop:Jtilde}) since $\tilde J M = 0$. Recall that the
effect of $(\I-M)$ is to first project on a vector with zero local
sum, $\sum_k\bigl(\v(\I-M)\bigr)_{ai,k} = 0,\ \forall i\in\V$, so we
assume directly $\v$ of the form
\[
v_{ai,k} = x_{ai,k}b_k^{(i)},\qquad\text{with}\qquad 
\langle x_{ai}\rangle_{b^{(i)}} = 0.
\]
As a result $\v' = \v J^n = \v'(\I-M)$ is of the same form. Let
$x'_{a'j,\ell}\egaldef v'_{a'j,\ell}/b_\ell^{(j)}$. We have
\[
\|x'\|_{\pi,b} 
\le \sum_{a'\to j}\pi_{a'j}\sum_{a\to i}\bigl(A^n\bigr)_{ai}^{a'j}
\|y_{a'j}\|_{b^{(j)}}
\]
with $y_{a'j,\ell}\ b_\ell^{(j)} = \sum_k
x_{ai,k}b_k^{(i)}\bigl(B_{ai,a'j}\n\bigr)_{k\ell}$.
From Lemma~\ref{lem:ineq2} applied to $y_{a'j}$,
\begin{align*}
\|x'\|_{\pi,b} 
&\le
\sum_{a'\to j}\pi_{a'j}\sum_{a\to i}\bigl(A^n\bigr)_{ai}^{a'j}\mu_2^n\|x_{ai}\|_{b^
{(i)}}
= \lambda_1^n\mu_2^n \|x\|_{\pi,b},
\end{align*}
since $\boldsymbol\pi$ is the right Perron vector of $A$. 
\end{proof}

\bibliography{refer}

\begin{thebibliography}{17}
\providecommand{\natexlab}[1]{#1}
\providecommand{\url}[1]{\texttt{#1}}
\expandafter\ifx\csname urlstyle\endcsname\relax
  \providecommand{\doi}[1]{doi: #1}\else
  \providecommand{\doi}{doi: \begingroup \urlstyle{rm}\Url}\fi

\bibitem[Berge(1967)]{Berge}
C.~Berge.
\newblock \emph{Th\'eorie des graphes et ses applications}, volume~II of
  \emph{Collection Universitaire des Math{\'e}matiques}.
\newblock Dunod, 2{\`e}me edition, 1967.

\bibitem[Br{\'e}maud(1999)]{Bremaud}
P.~Br{\'e}maud.
\newblock \emph{Markov chains: Gibbs fields, Monte Carlo simulation and
  queues.}
\newblock Springer-Verlag, 1999.

\bibitem[Diaconis and Strook(1991)]{DiaStr}
P.~Diaconis and D.~Strook.
\newblock Geometric bounds for eigenvalues of markov chains.
\newblock \emph{Ann. Appl. Probab}, 1\penalty0 (1):\penalty0 36--61, 1991.

\bibitem[Furtlehner et~al.(2010)Furtlehner, Lasgouttes, and Auger]{FuLaAu}
C.~Furtlehner, J.-M. Lasgouttes, and A.~Auger.
\newblock Learning multiple belief propagation fixed points for real time
  inference.
\newblock \emph{Physica A: Statistical Mechanics and its Applications},
  389\penalty0 (1):\penalty0 149--163, 2010.

\bibitem[Halmos(1974)]{Halmos}
P.~R. Halmos.
\newblock \emph{Finite-Dimensional Vector Space}.
\newblock Springer-Velag, 1974.

\bibitem[Hartfiel(1997)]{Hart}
D.~J. Hartfiel.
\newblock System behavior in quotient systems.
\newblock \emph{Applied Mathematics and Computation}, 81\penalty0 (1):\penalty0
  31--48, 1997.

\bibitem[Heskes(2003)]{Heskes4}
T.~Heskes.
\newblock Stable fixed points of loopy belief propagation are minima of the
  {B}ethe free energy.
\newblock \emph{Advances in Neural Information Processing Systems}, 15, 2003.

\bibitem[Ihler et~al.(2005)Ihler, Fischer, and Willsky]{Ihler}
A.~Ihler, J.~I. Fischer, and A.~Willsky.
\newblock Loopy belief propagation: Convergence and effects of message errors.
\newblock \emph{J. Mach. Learn. Res.}, 6:\penalty0 905--936, 2005.

\bibitem[Kschischang et~al.(2001)Kschischang, Frey, and Loeliger]{Kschi}
F.~R. Kschischang, B.~J. Frey, and H.~A. Loeliger.
\newblock Factor graphs and the sum-product algorithm.
\newblock \emph{IEEE Trans. on Inf. Th.}, 47\penalty0 (2):\penalty0 498--519,
  2001.

\bibitem[Mooij and Kappen(2007)]{MooijKappen07}
J.~M. Mooij and H.~J. Kappen.
\newblock Sufficient conditions for convergence of the sum-product algorithm.
\newblock \emph{IEEE Trans. on Inf. Th.}, 53\penalty0 (12):\penalty0
  4422--4437, 2007.

\bibitem[Pearl(1988)]{Pearl}
J.~Pearl.
\newblock \emph{Probabilistic Reasoning in Intelligent Systems: Network of
  Plausible Inference}.
\newblock Morgan Kaufmann, 1988.

\bibitem[Seneta(2006)]{Sen}
E.~Seneta.
\newblock \emph{Non-negative matrices and Markov chains}.
\newblock Springer, 2006.

\bibitem[Tatikonda and Jordan(2002)]{Tatikonda02}
S.~Tatikonda and M.~Jordan.
\newblock Loopy belief propagation and gibbs measures.
\newblock In \emph{UAI-02}, pages 493--50, 2002.

\bibitem[Wainwright(2002)]{Wain}
M.~J. Wainwright.
\newblock \emph{Stochastic processes on graphs with cycles: geometric and
  variational approaches}.
\newblock PhD thesis, MIT, Jan. 2002.

\bibitem[Watanabe and Fukumizu(2009)]{WaFu}
Y.~Watanabe and K.~Fukumizu.
\newblock Graph zeta function in the bethe free energy and loopy belief
  propagation.
\newblock In \emph{Advances in Neural Information Processing Systems},
  volume~22, pages 2017--2025, 2009.

\bibitem[Weiss(2000)]{Weiss}
Y.~Weiss.
\newblock Correctness of local probability propagation in graphical models with
  loops.
\newblock \emph{Neural Computation}, 12\penalty0 (1):\penalty0 1--41, 2000.

\bibitem[Yedidia et~al.(2005)Yedidia, Freeman, and Weiss]{YeFrWe3}
J.~S. Yedidia, W.~T. Freeman, and Y.~Weiss.
\newblock Constructing free-energy approximations and generalized belief
  propagation algorithms.
\newblock \emph{IEEE Trans. Inform. Theory.}, 51\penalty0 (7):\penalty0
  2282--2312, 2005.

\end{thebibliography}
\bibliographystyle{abbrvnat}
\end{document}